\documentclass[runningheads]{llncs}

\usepackage{hyperref}
\usepackage{graphicx}
\usepackage{amsmath,mathtools}
\usepackage{amssymb}
\usepackage{xspace}
\usepackage{verbatim}
\usepackage[all]{xy}
\usepackage{afterpage}
\usepackage{placeins}

\title{On the Ontological Modeling of Trees} 
\titlerunning{On the Ontologial Modeling of Trees}

\author{David Carral\inst{1} \and Pascal Hitzler\inst{2} \and Hilmar Lapp\inst{3} \and Sebastian Rudolph\inst{1}}
\institute{TU Dresden, Germany \and Data Semantics (DaSe) Laboratory,
  Wright State University, OH, USA \and Center for Genomic and
  Computational Biology, Duke University, Durham, NC, USA}

\authorrunning{David Carral, Pascal Hitzler, Hilmar Lapp, Sebastian Rudolph}

\begin{document}

\maketitle

\begin{abstract}
Trees -- i.e., the type of data structure known under this name -- are central to many aspects of knowledge organization. We investigate some central design choices concerning the ontological modeling of such trees. In particular, we consider the limits of what is expressible in the Web Ontology Language, and provide a reusable ontology design pattern for trees.  
\end{abstract}


\section{Introduction}\label{sec:intro}


Trees are fundamental data structures for knowledge organization. They make their appearance in the form of taxonomies, meronomies, decision trees, branching processes, etc. As such they are fundamental for ontological knowledge representation. 

At the same time, however, it is not possible to fully characterize trees in the Web Ontology Language (OWL) \cite{owl2-primer,FOST} (see Section \ref{sec:tree-limits}). It is thus an important research question how to represent trees in ontology modeling, and to understand the pros and cons of different ways to do it. 

We need to realize, of course, that trees in ontology modeling often serve a different purpose than in programming. Operations on trees important in programming include, for example, adding or deleting items or pruning of whole sections; i.e., some of the important operations do actually change the tree. For ontology modeling purposes, in contrast, it is more appropriate to think of a tree as static and as something which is being queried. Typical queries would be to identify roots or leaves, common ancestors, or descendants. 

However, despite the importance of trees for knowledge organization, there is currently no corresponding ontology design pattern available on ontologydesignpatterns.org. In this paper, we will provide such a pattern, and will also discuss different design choices as well as their respective advantages and disadvantages. 

The rest of the paper is structured as follows. In Section \ref{sec:trees-of-life} we present a particularly interesting use case which has informed our work, namely the use of ontology modeling for evolutionary or phylogenetic trees. In Section \ref{sec:tree-limits} we  discuss the fundamental shortcomings of the Web Ontology Language (OWL) regarding the modeling of trees.\footnote{The pattern is available from \url{http://ontologydesignpatterns.org/wiki/Submissions:Tree_Pattern}} In Section \ref{sec:arb-trees} we present a basic ontology design pattern for the modeling of trees. In Section \ref{sec:bin-trees} we discuss the special case of $n$-bounded trees (e.g., with $n = 2$ for binary trees). In Section \ref{sec:conclusion} we conclude. 

\section{Phylogenetic Trees}\label{sec:trees-of-life}

One of the central tenets following from the theory of organismal
evolution is that all life is related through descent with
modification \cite{Darwin1859-ka}. That is, populations of a
biological species can over time diverge enough, due to natural
selection, adaptation, genetic drift, and other forces acting
differentially on different populations, that they form new species,
some of which persist and go on themselves to split, giving rise to
new species, and so forth. Speciation through diversification can
sometimes be driven by new ecologic opportunities, for example when
new habitats are being colonized, a process often referred to as
adaptive radiation \cite{Simpson1949-pv,Simpson1953-qg}. One of the
most prominent research objectives in evolutionary science is to
reconstruct, using genetic and organismal trait data, the evolutionary
history of different organisms, species, or life forms; i.e., to
reconstruct the lines of shared descent by which organisms are
connected \cite{Felsenstein2003-ji,Swofford1996-ya}. Such a
reconstruction is represented in the form of a phylogenetic tree, in
which the leaves are often called operational taxonomic units (OTUs)
and represent the sampled entities, and internal nodes represent
ancestral entities, such as ancestral populations from which
descendent ones diverged. Phylogenetic reconstruction results in
unrooted trees; the root is normally not known (and cannot normally be
sampled), but reasonably accurate mechanisms for inducing a root exist
\cite{Huelsenbeck2002-hm,Maddison1984-qa} (for example, by including
in the reconstruction analysis a group of species -- a so-called
``outgroup'' -- that are already known to fall outside of the ingroup
for which evolutionary patterns are being studied).

A phylogenetic tree represents important evolutionary hypotheses about\linebreak[4]
shared history. For example, two OTUs A and B are more closely related
to each other than to OTU C if A and B share a more recent common
ancestor than they do with C. The subtree descending from a node forms
a clade, clades which share a parent are called sister clades. One of
the major objects of comparative phylogenetics is to identify the
properties and processes (organismal traits, geographic range, tempo and mode of
evolution, etc) by which one clade differs from others, in particular
its sisters, and how these properties change along lines of descent in
the tree \cite{Felsenstein1985-au,OMeara2011-br}. This gives rise to a
number of important queries when mapping data onto phylogenetic trees
for (or as a result of) analysis. Particularly ubiquitous operations
on trees include the following: (1) finding the most recent common
ancestor of a given number of nodes (usually leaf nodes); (2)
enumerating the leaf nodes, or all nodes descending from a given
(internal) node; (3) enumerating the sequence of ancestors of a node
to the root; and (4) identifying the last ancestor of a node A from
which another node B is not also descended. We will come back to these and other operations as part of the competency questions for our modeling in Section \ref{sec:arb-trees}.

Operations (1) and (4) correspond to two principle ways in which the
semantics of clade concepts can be defined on a tree
\cite{De_Queiroz1990-bv}, whether using a concrete instantiation of a
tree, or a hypothetical one. In the field of phylogenetic taxonomy
\cite{De_Queiroz1992-oq}, a clade concept defined by the most recent
common ancestor of a set of (usually leaf) nodes includes the common
ancestor and is referred to as a node-based definition. In contrast, a
branch-based definition circumscribes the clade as the last ancestor
of a (usually leaf) node that excludes (i.e., does not have as a
descendant) another node (also usually a leaf node). The semantics of
a clade concept defined in this way is such that the branch subtending
from the ancestor node to its parent is included (hence the name
“branch-based”). To understand this, remember that a phylogenetic tree
is a model of evolutionary lines of descent reconstructed from sampled
data. In reality, there may be lines of descent which were not
observed (sampled), for example because all organisms from those lines
are now extinct, but which, had they been observed, would originate
from the subtending branch and which would therefore still be included
in the clade because they would branch off after the lineage to be
excluded.

It is worth noting that spurred in part by the exponentially
increasing amount of data available for phylogenetic reconstruction,
very large trees encompassing up to tens of thousands of taxa have
recently become available \cite{Goloboff2009-oa,Driskell2004-wq,Dunn2008-tr,Smith2009-yk,Jarvis2014-lq}, culminating in the initial publication of
the synthesized Open Tree of Life with about 2 million tips \cite{Hinchliff2015-nd}. Such
encompassing trees open up unprecedented opportunities for comparative
phylogenetic research. However, this also means that our knowledge
about the evolution of life is changing at increasing pace and
breadth, which makes it necessary to efficiently map clade definitions
from one tree to another, or from one revision of the Open Tree of
Life to a future one. A recent initiative, termed ``phyloreferencing''
(http://phyloref.org) aims to accomplish this by using machine
reasoning over ontological representations of the semantics of both
clade definitions and phylogenetic trees \cite{Cellinese2015-zv,Michael_Keesey2007-lb,Sereno2005-wv}. In the rest of this paper, we abstract from the specific use case and look at the task of ontological modeling of trees in general.

\section{Fundamental Limitations Regarding Tree Modeling}\label{sec:tree-limits}

In order to investigate to what degree tree-based properties can be expressed using common KR formalisms, we first need to formally define what structures we denote by the notion ``tree''.

\begin{definition}\label{def:tree} A \emph{rooted directed branching tree} (short: \emph{tree}) is defined as a directed graph $T=(V,E)$ where $V$ is a set called \emph{vertices} or \emph{nodes} and $E\subseteq V \times V$ is the set of edges, satisfying the following properties:
\begin{enumerate}
\item There is exactly one node $r \in V$ called \emph{root}, which has no incoming edges, i.e., $E \cap (V \times \{r\}) = \emptyset$.
\item Every node $v \in V \setminus \{r\}$ that is not the root has exactly one incoming edge, i.e., there exists exactly one $v' \in V$ such that $(v',v)\in E$. We then call $v'$ the parent of $v$ and $v$ the child of $v'$.
\item Every node $v \in V$ can be reached from the root traversing edges, i.e., there is a number $n\geq 0$ and a sequence $(v)_{i\in \{0,\ldots,n\}}$ such that $r=v_0$, $v=v_n$, and for all $i\in \{0,\ldots,n-1\}$ we have $(v_i,v_{i+1})\in E$.
\end{enumerate}
A node without children will be called \emph{leaf {node}}. A \emph{binary tree} is a tree where every node that is not a leaf has exactly two children. An \emph{$n$-ary tree} is a tree where every node that is not a leaf has exactly $n$ children. An \emph{$n$-bounded tree} is a tree where every node has at most $n$ children.  A tree is \emph{finite} if $V$ is finite.
\end{definition}

When modeling trees using some logic-based KR language, we would like to achieve that we can create a knowledge base which has exactly all (finite) trees as its models (possibly using additional auxiliary vocabulary), in other words, we would like to characterize or axiomatize the class of all (finite) trees.

Unfortunately, it is not too hard to show that this is not possible by any KR formalism that is expressible in first order predicate logic (FOL). A very helpful tool for showing this is the well-known compactness theorem of first-order logic\cite{compactness}.

\begin{theorem}[Compactness of FOL]
A set $\Phi$ of FOL sentences is satisfiable if and only if every finite subset of $\Phi$ is.	
\end{theorem}

We now use this theorem to show our negative result.

\begin{proposition}
Let $\psi$ be a FOL sentence (using the binary predicate ``$edge$'') such that every finite tree $T=(V,E)$ corresponds to some model $\mathcal{I}$ of $\psi$, i.e., $(V,E) \cong (\Delta^\mathcal{I},edge^\mathcal{I})$. Then, $\psi$ also has a model which does not correspond to any (finite or infinite) tree.
\end{proposition}

\begin{proof}
Consider the following sequence $(\varphi_i)_{i\in \mathbb{N}}$ of FOL sentences (where $a$ is a fresh constant):
\begin{itemize}
\item[] $\varphi_1 := \exists x_1. edge(x_1,a)$
\item[] $\varphi_2 := \exists x_1 \exists x_2. edge(x_2,x_1) \wedge edge(x_1,a)$
\item[] $\varphi_3 := \exists x_1\exists x_2\exists x_3. edge(x_3,x_2) \wedge edge(x_2,x_1) \wedge edge(x_1,a)$
\item[] $\vdots$
\end{itemize} 

In words, $\varphi_k$ expresses that the node $a$ has an incoming edge-path of length $k$.
Now let $\Phi := \{\psi\} \cup \{\varphi_k \mid k \in \mathbb{N}\}$. Obviously, every finite subset of $\Phi$ is satisfiable (intuitively, just pick an arbitrary large finite tree and then pick $a$ such that it is ``deep enough'' in the tree). Then, by compactness of FOL, $\Phi$ itself must be satisfiable. However, in a model $\mathcal{I}$ of $\Phi$ the element $a^\mathcal{I}$ cannot be reachable from the root, since then it would have an incoming edge-path of maximal length which cannot be the case by construction of $\Phi$. Hence $\mathcal{I}$ cannot correspond to a tree. By construction, $\mathcal{I}$ is also a model of $\psi$.\qed
\end{proof}

This result shows, that trees (finite or infinite) are not fully axiomatizable in FOL and any attempt to do so will only be approximate (although practically useful). 

On the other hand, trees are axiomatizable when we extend FOL (or just DLs for that matter) by a transitive closure operator for binary predicates. Assume that, for every binary predicate (or in DL terms: role) $p$, we allow for a binary predicate/role name $p^+$ and define its semantics such that $(p^+)^\mathcal{I}$ is the transitive closure of $p^\mathcal{I}$. Then the conditions of Definition~\ref{def:tree} can be expressed using the following axioms:
\begin{align}
\{\text{root}\} &\sqsubseteq  \neg \exists \text{edge}^-.\top\\
\neg \{\text{root}\} &\sqsubseteq  {=}1\, \text{edge}^-.\top \\
\neg \{\text{root}\} &\sqsubseteq \exists (\text{edge}^-)^+.\{\text{root}\}
\end{align}
To axiomatize the class of binary trees, the following axiom can be added: 
\begin{align}
\top &\sqsubseteq  \neg \exists \text{edge}.\top \sqcup {{=}2} \,\text{edge}.\top
\end{align}
In order to impose finiteness, one can axiomatize (as an auxiliary additional structure) a finite linear order with a starting element and an ending element and the successor role:

\noindent
\begin{minipage}{0.49\textwidth}
\begin{align}
\{\text{start}\} &\sqsubseteq  \neg \exists \text{succ}^-.\top\\
\neg \{\text{start}\} &\sqsubseteq  {=}1\, \text{succ}^-.\top \\
\neg \{\text{start}\} &\sqsubseteq \exists (\text{succ}^-)^+.\{\text{start}\}
\end{align}
\end{minipage}
\hfill
\begin{minipage}{0.49\textwidth}
\begin{align}
\{\text{end}\} &\sqsubseteq  \neg \exists \text{succ}.\top\\
\neg \{\text{end}\} &\sqsubseteq  {=}1\, \text{succ}.\top \\
\neg \{\text{end}\} &\sqsubseteq \exists (\text{succ})^+.\{\text{end}\}
\end{align}
\end{minipage}

\bigskip

Transitive closures, of course, are not part of the OWL specification \cite{owl2-primer}, i.e., this characterization cannot be used when modeling in the Web Ontology Language. Description logics featuring regular expressions over roles have, however, been considered since the early days of DL research \cite{DBLP:conf/ijcai/Baader91} and decision and query answering procedures have been described for very expressive DLs with that feature \cite{DBLP:conf/ijcai/CalvaneseEO09}.

\section{A Simple Tree Pattern}\label{sec:arb-trees}

The repository of ontology design patterns on ontologydesignpatterns.org does not contain any pattern for trees. There is also none for graphs which could have been used to specialize a trees pattern. The repository contains a pattern for lists, though,\footnote{\url{http://ontologydesignpatterns.org/wiki/Submissions:List}} and a list pattern could be generalizable to a tree pattern. 

The schema diagram for this list pattern is depicted in Figure~\ref{fig:list}. It reuses the sequence pattern\footnote{\url{http://ontologydesignpatterns.org/wiki/Submissions:Sequence}} which seems to be the relevant part for our purposes. We depict the sequence pattern schema diagram in Figure \ref{fig:sequence} and all non-tautological axioms are given in Figure \ref{fig:sequenceAxioms}.\footnote{Generated with the OWLAPI \LaTeX{} renderer \cite{eswc17-latex}.} The axiomatization appears to be rather minimalistic, e.g., ``follows'' should be transitive over ``directlyFollows'', and for a sequence we should also use cardinality restrictions to limit the number of followers and predecessors. We will return to this later on.

\begin{figure}[t]
\begin{center}
\includegraphics[width=\textwidth]{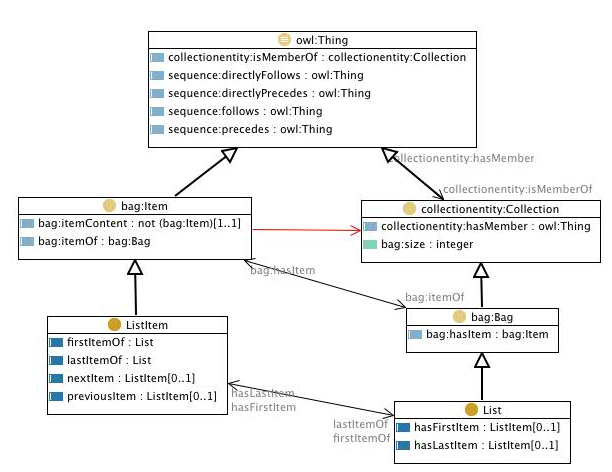}
\end{center}
\caption{List pattern schema diagram from ontologydesignpatterns.org}\label{fig:list}
\end{figure}

\begin{figure}[t]
\begin{center}
\includegraphics[width=\textwidth]{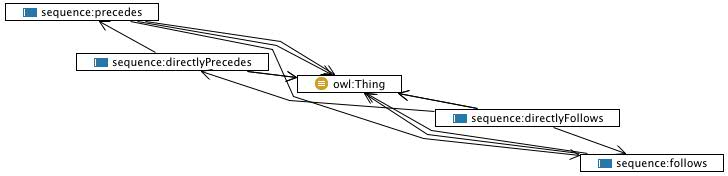}
\end{center}
\caption{Sequence pattern schema diagram from ontologydesignpatterns.org}\label{fig:sequence}
\end{figure}

\begin{figure}[t]
\begin{align*}
\text{directlyFollows} &\sqsubseteq  \text{follows}\\
\text{directlyFollows} &\equiv  \text{directlyPrecedes}^- \\
\text{directlyPrecedes} &\sqsubseteq  \text{precedes}\\
\text{precedes} &\equiv  \text{follows}^- \\
\text{TransitiveProperty}&(\text{follows})\\
\text{TransitiveProperty}&(\text{precedes})\\
\end{align*}
\caption{Axioms for the sequence pattern from Figure \ref{fig:sequence}. We omitted axioms that were tautologies.}\label{fig:sequenceAxioms}
\end{figure}

The list pattern just cited provides basic building blocks for a simple tree pattern. However, we opt to change the names of the properties: It seems to be more appropriate to use ``hasChild'' and ``hasDescendant'' rather than ``directlyPrecedes'' and ``precedes'', and to use ``hasParent'' and ``hasAncestor'' rather than ``directlyFollows'' and ``follows.''

Before proceeding with the tree pattern, we present a set of competency questions \cite{chess-odp-book} which seem representative to us and include operationes raised as important in Section \ref{sec:trees-of-life}:
\begin{enumerate}
\item\label{cq1} Determine the root.
\item\label{cq2}  Determine all ancestors of a given node.
\item\label{cq3}  Determine all leaves.
\item\label{cq4}  Determine all descendants of a given node.
\item\label{cq5}  Determine all descendants of a given node which are leaves.
\item\label{cq6}  Given two nodes, determine whether one is a descendant of the other. 
\item\label{cq7} Given two nodes, determine all commmon ancestors.
\item\label{cq8}  Given two nodes, determine the latest common ancestor.
\item\label{cq9}  Given two nodes $x$ and $y$, determine the earliest ancestor of $x$ which is not an ancestor of $y$.
\end{enumerate}

We next give our proposal for a simple tree pattern. Afterwards we will discuss our design choices. The schema diagram is given in Figure \ref{fig:tree}. However, the axiomatization is really much more important; it can be found in Figure \ref{fig:tree-axioms}. Note that some additional desired axioms, such as $\text{hasParent} \sqsubseteq \text{hasAncestor}$ and transitivity of hasAncestor can be inferred from the ones stated.

\begin{figure}[t]
\begin{center}
\includegraphics[width=.6\textwidth]{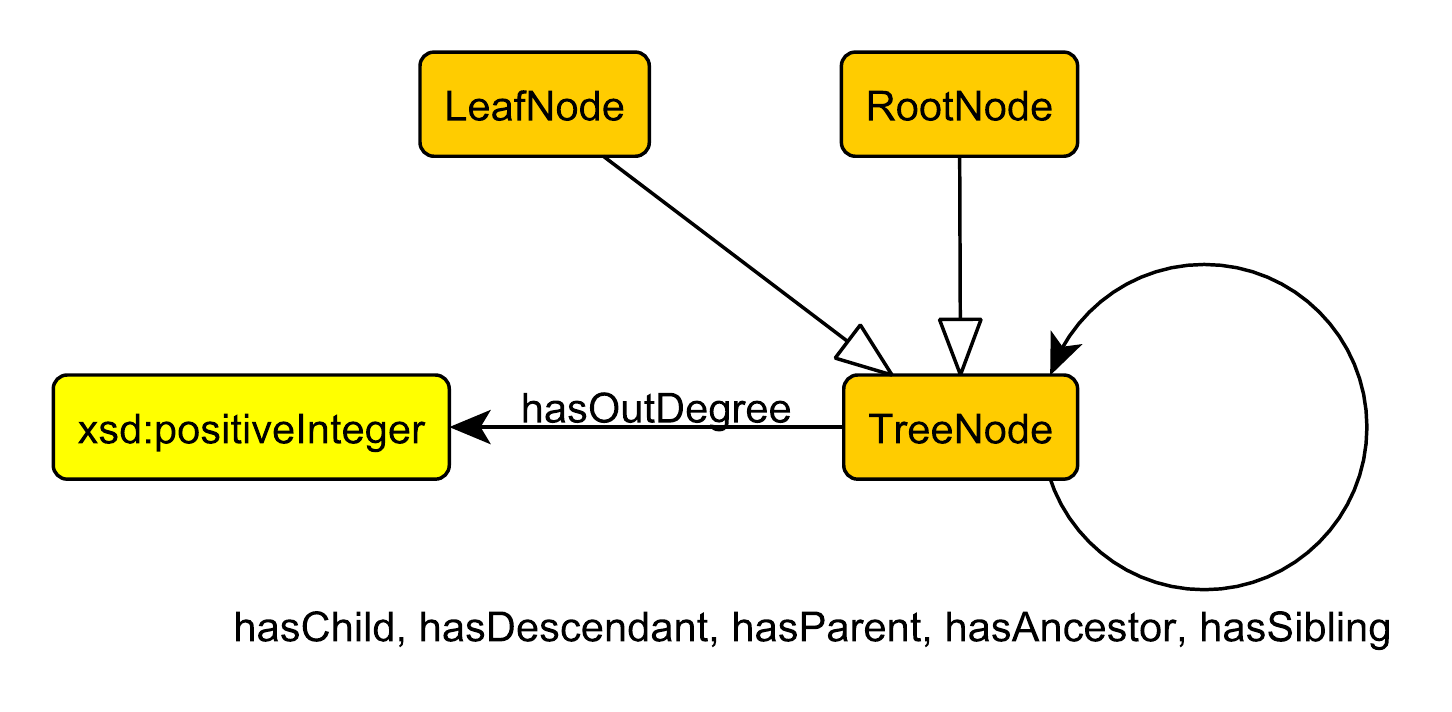}
\end{center}
\caption{Schema diagram for the simple tree pattern. Unlabelled arrows are subclass relationships.}\label{fig:tree}
\end{figure}

\begin{figure}[t]
\begin{align*}
\text{RootNode}&(a) &
\text{LeafNode}&(d) &
\text{LeafNode}&(e) \\
\text{LeafNode}&(f)&
\text{TreeNode}&(b) &
\text{TreeNode}&(c)\\
\text{hasChild}&(a,b) &
\text{hasChild}&(a,c) &
\text{hasChild}&(b,d) \\
\text{hasChild}&(b,e)&
\text{hasChild}&(c,f) &
\text{hasSibling}&(b,c)\\ 
\text{hasSibling}&(d,e) &
\text{hasOutDegree}&(a,2) &
\text{hasOutDegree}&(b,2)\\
\text{hasOutDegree}&(c,1) &
\text{hasOutDegree}&(d,0) &
\text{hasOutDegree}&(e,0)\\
\text{hasOutDegree}&(f,0)
\end{align*}
\caption{Example ABox for a tree.}\label{fig:tree-example}
\end{figure}

\begin{figure}[t]\small
\begin{align}
\text{LeafNode} &\sqsubseteq \text{TreeNode}\label{ax1}\\
\text{RootNode} &\sqsubseteq \text{TreeNode}\\
\text{TreeNode} &\sqsubseteq \forall\text{hasOutDegree}.\text{xsd:positiveInteger}\\
\text{TreeNode} &\sqsubseteq \mathord{=}1\text{hasOutDegree}.\text{xsd:positiveInteger}\\
\text{LeafNode} &\equiv \text{TreeNode} \sqcap \phantom{x}\notag\\ &\phantom{\equiv} \quad\forall\text{hasOutDegree}.\{0^{\wedge\wedge}\text{xsd:positiveInteger}\}\\
\text{TreeNode} \sqcap \neg\text{LeafNode} &\equiv \text{TreeNode} \sqcap \phantom{x} \notag\\ &\phantom{\equiv}\quad\forall\text{hasOutDegree}.\{x^{\wedge\wedge}\text{xsd:positiveInteger}\mid 1\leq x\}\label{ax6}\\
\text{hasChild} &\equiv \text{hasParent}^-\\
\text{hasDescendant} &\equiv \text{hasAncestor}^-\\
\text{hasChild} &\sqsubseteq \text{hasDescendant}\\
\text{hasDescendant} \circ \text{hasDescendant} &\sqsubseteq \text{hasDescendant}\\
\text{TreeNode} &\sqsubseteq \forall\text{hasChild}.\text{TreeNode}\\
\text{TreeNode} \sqcap \neg \text{LeafNode} &\equiv \text{TreeNode} \sqcap \exists \text{hasChild}.\text{TreeNode}\\
\text{TreeNode} &\sqsubseteq \forall\text{hasDescendant}.\text{TreeNode}\\
\text{TreeNode} &\sqsubseteq \forall\text{hasParent}.\text{TreeNode}\\
\text{TreeNode} &\sqsubseteq \forall\text{hasSibling}.\text{TreeNode}\\
\text{TreeNode} \sqcap \neg \text{RootNode} &\equiv \text{TreeNode} \sqcap \mathord{=} 1\text{hasParent}.\top\\
\text{TreeNode} &\sqsubseteq \forall\text{hasAncestor}.\text{TreeNode}\\
\text{RootNode} &\equiv \text{TreeNode} \sqcap \neg \exists\text{hasParent}.\top\\
\text{LeafNode} &\equiv \text{TreeNode} \sqcap \neg \exists\text{hasChild}.\top\\
\text{Irreflexive}&(\text{hasChild})\\
\text{Irreflexive}&(\text{hasParent})\\
\text{Irreflexive}&(\text{hasDescendant})\label{ax25}\\
\text{Irreflexive}&(\text{hasAncestor})\label{ax26}\\
\text{hasSibling} &\equiv \text{hasSibling}^-\\
\text{Irreflexive}&(\text{hasSibling})
\end{align}
\caption{Axioms for the tree pattern from Figure \ref{fig:tree}.}\label{fig:tree-axioms}
\label{figure:tree}
\end{figure}

Before we proceed, let us first make a concrete example how this pattern informs the graph structure of the ABox. Given a tree such as 
$$
\xymatrix{
& & a \ar[dl] \ar[dr] & &\\
& b \ar[dl] \ar[dr] & & c \ar[dr]&\\
d & & e & & f
}
$$
we encode it using the ABox from Figure \ref{fig:tree-example}; note that we omit redundant statements which can be inferred from the axioms. 

The axioms from Figure \ref{fig:tree-axioms} should be self-explanatory. Axiom (\ref{ax6}) uses a datatype facet. The complete set of axioms is not in OWL DL because axioms (\ref{ax25}) and (\ref{ax26}) declare irreflexivity for non-simple roles. If it is desireable to stay within OWL DL, these axioms could be omitted.

Our pattern and axiomatization include some terms which may appear to be redundant. Indeed, we could have omitted the use of hasParent and hasAncestor as these are simply inverses of hasChild and hasDescendant, respectively. Other aspects, however, are not redundant.

The property hasOutDegree may appear to be redundant, as it captures the number of children of a node. However, it is not redundant as far as the OWL model is concerned, because the underlying open world assumption makes it impossible to count the number of children. This could of course be addressed using a (local) closed world extension of OWL, such as in \cite{CarralJH12}, however the current standard does not support this. For the same reason, membership in the classes RootNode and LeafNode cannot be inferred.

\afterpage{\FloatBarrier}

The case of the hasSibling property is more intricate. Again it would naively appear as if it were redundant. However, we have not been able yet to axiomatize it in the general case; for special cases where it is possible see Section \ref{sec:bin-trees}. A naive attempt by means of a rule\footnote{This rule can be converted into OWL DL using rolification, see \cite{KrisnadhiMH11}.} 
$$\text{hasParent}(x,y) \wedge \text{hasChild}(y,z) \to \text{hasSibling}(x,z)$$ 
is insufficient because for addressing some competency questions we require irreflexivity of hasSibling, while the rule above renders hasSibling to be non-simple, which is not allowed together with irreflexivity in OWL DL.

Let us return to the competency questions listed earlier; it turns out we can address them all even when omitting the irreflexivity axioms for hasDescendant and hasAncestor, such that our model stays within OWL DL. Questions \ref{cq1} and \ref{cq3} can be addressed using the RootNode and LeafNode classes. Questions \ref{cq2} and \ref{cq4} are straightforward, as is question \ref{cq5} using the LeafNode class. Question \ref{cq6} can be solved with two queries using the hasDescendant property. 

The remaining questions are more intricate. Given two nodes $x$ and $y$, Question \ref{cq7} can be addressed via
$$\exists\text{hasDescendant}.\{x\} \sqcap \exists\text{hasDescendant}.\{y\}.$$
Question \ref{cq8} seems to require use of the hasSibling property. For readability we first give the solution as first-order predicate logic formula: 
\begin{align*}
&\phantom{\wedge} \text{hasChild}(z,w) \wedge (w=x \vee \text{hasDescendant}(w,x))\\ &\wedge \text{hasSibling}(w,v) \wedge (v=y \vee \text{hasDescendant}(v,y)).
\end{align*}
Conversion of this into OWL following the approach laid out in \cite{KrisnadhiMH11} results in the class description 
\begin{align*}
\exists\text{hasChild}.(&(\{x\}\sqcup\exists\text{hasDescendant}.\{x\})\\
&\sqcap (\exists\text{hasSibling}.(\{y\}\sqcup\exists\text{hasDescendant}.\{y\})).
\end{align*}

In a similar fashion, Question \ref{cq9} can be addressed using the class description 
$$(\{x\}\sqcup\exists\text{hasDescendant}.\{x\})\sqcap (\exists\text{hasSibling}.(\{y\}\sqcup\exists\text{hasDescendant}.\{y\})).$$

Note that irreflexivity of hasSibling is required for the class descriptions just given. Or to be more precise, what is required is that there is no node $z$ in the tree for which $\text{hasSibling}(z,z)$ is declared or can be inferred; note that this is in fact a weaker requirement than what we get by declaring irreflexivity. As a consequence, the irreflexivity declaration for hasSibling can actually be omitted from the axiomatization without impact on the just given solution to Question~\ref{cq9}; however we prefer to keep the irreflexivity declaration in the axiomatization as it disambiguates the model \cite{HitzlerK16}.   As mentioned earlier, we have not been able to find a solution to infer the (irreflexive) hasSibling relationship in the general case using OWL axioms, thus we require it as a primitive. We will revisit this in the next section, though.

Let us finally use our tree pattern to recover a list pattern based on it. The schema diagram can be found in Figure \ref{fig:list-new}. We keep only one property, hasNext, which corresponds to hasChild, and hasSuccessor, which corresponds to hasDescendant. The root becomes the first list item, the leaves become the last list item. The outdegree is always 1 unless it's the last list item, so we also omit this information. The corresponding axiomatization, as derived from the tree axiomatization above, can be found in Figure \ref{fig:list-axioms-new}. As before, Axiom (\ref{ax-list-irreflexive}) causes the pattern to fall outside OWL DL, and if this is undesirable, this axiom should be omitted.

\begin{figure}[t]
\begin{center}
\includegraphics[width=.4\textwidth]{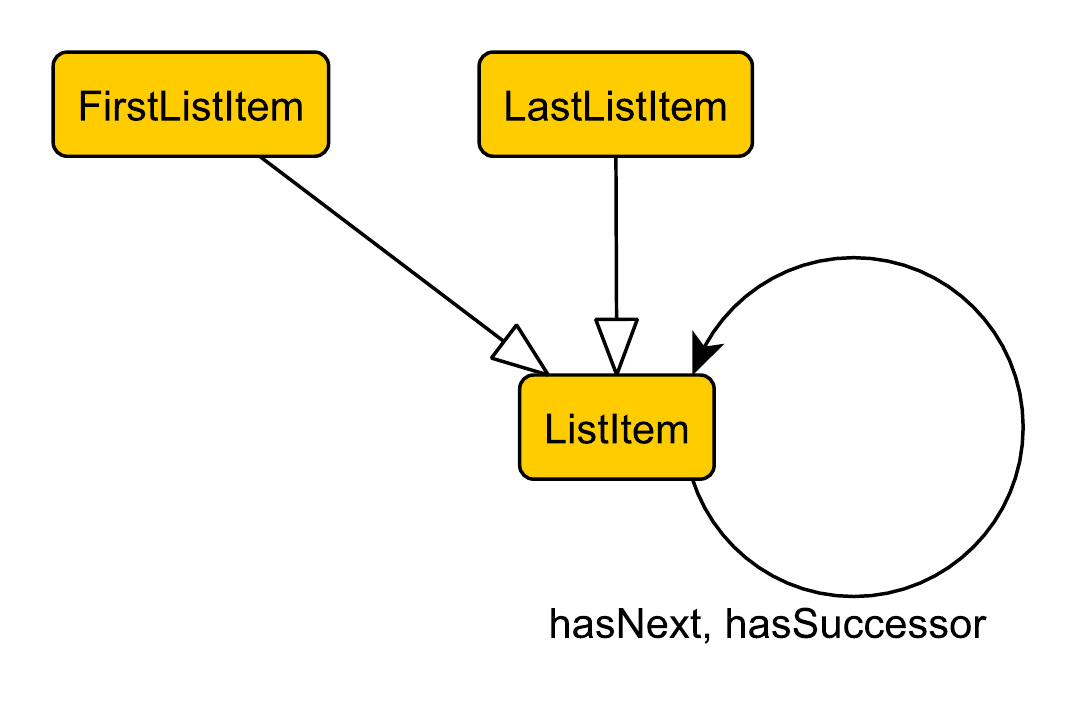}
\end{center}
\caption{Schema diagram for the simple list pattern, derived from the tree pattern. }\label{fig:list-new}
\end{figure}

\begin{figure}[t]\small
\begin{align}
\text{FirstListItem} &\sqsubseteq \text{ListItem}\\
\text{LastListItem} &\sqsubseteq \text{ListItem}\\
\text{ListItem} &\sqsubseteq \forall\text{hasNext}.\text{ListItem}\\
\text{ListItem} &\sqsubseteq \forall\text{hasNext}^-.\text{ListItem}\\
\text{ListItem} \sqcap \neg \text{LastListItem} &\equiv \text{ListItem} \sqcap \mathord{=}1\text{hasNext}.\text{ListItem}\\
\text{ListItem} \sqcap \neg \text{FirstListItem} &\equiv \text{ListItem} \sqcap \mathord{=}1\text{hasNext}^-.\text{ListItem}\\
\text{FirstListItem} &\equiv \text{ListItem} \sqcap \neg \exists\text{hasNext}^-.\top\\
\text{LastListItem} &\equiv \text{ListItem} \sqcap \neg \exists\text{hasNext}.\top\\ 
\text{hasNext} &\sqsubseteq \text{hasSuccessor}\\
\text{hasNext} \circ \text{hasSuccessor} &\sqsubseteq \text{hasSuccessor}\\
\text{Irreflexive}&(\text{hasSuccessor})\label{ax-list-irreflexive}
\end{align}
\caption{Axioms for the lists pattern from Figure \ref{fig:list-new}.}\label{fig:list-axioms-new}
\end{figure}

\section{Trees With Bounded Arity}\label{sec:bin-trees}

In this section, we look at $n$-bounded trees as a special case, and present a set of OWL axioms that can be employed to model these.

\begin{definition}
A \emph{rooted directed branching $n$-bounded tree} (short: \emph{$n$-bounded tree}) is a tree $T = (V, E)$ where every node has at most $n$ outgoing edges; i.e., for every $v \in V$, $\vert E \cap (\{v\} \times V) \vert \leq n$.
\end{definition}

In the previous section, we have discussed why we needed to include an explicit hasSibling relation in our model, namely because we were unable to axiomatically define it in OWL. If we know that the trees under consideration are $n$-bounded, though, we can in fact infer the hasSibling relation. 

To this end, we introduce a set of additional axioms (see Figure \ref{figure:binaryTreeAxioms}) that, if combined with the axioms  from Figure \ref{figure:tree} can be used to axiomatize the structure of an $n$-bounded tree. Key to this 
are axioms (\ref{disjunction}) and (\ref{disjoint}), which together limit the number of children per node to a maximum of $n$. 

Note that, if these axioms are considered, we can indeed automatically infer the hasSibling relationship: Having an upper bound on the number of children per node, we can use a finite set of concept names $\text{Child}_i$ to differentiate the children of any given node in the tree.
Note how, due to (\ref{disjunction}) and (\ref{disjoint}), every $n$-ary tree node is typed by one and only one class $\text{Child}_i$.
Moreover, axioms in (\ref{atMost1Childi}) enforce that at most one child of every node is in the class $\text{Child}_i$ for every $i = 1, \ldots, n$.
Using axioms in (\ref{selfConnect}), we automatically infer that each child of every node is connected to itself via some property $R_i$ for some $i = 1, \ldots, n$.
Using axiom (\ref{chain}), 
we can infer the hasSibling relation.


\begin{figure}[t]
\begin{align}
n\text{-BoundedTreeNode} &\sqsubseteq \text{TreeNode} \\
n\text{-BoundedTreeNode} &\sqsubseteq \forall \text{hasAncestor}.n\text{-BoundedTreeNode} \\
n\text{-BoundedTreeNode} &\sqsubseteq \forall \text{hasDescendant}.n\text{-BoundedTreeNode} \\
n\text{-BoundedTreeNode} &\sqsubseteq \text{Child}_1 \sqcup \ldots \sqcup \text{Child}_n \label{disjunction} \\
\{\text{Child}_i \sqcap \text{Child}_j &\sqsubseteq \bot \mid 1 \leq i < j \leq n \} \label{disjoint} \\
\{n\text{-BoundedTreeNode} &\sqsubseteq \mathord{\leq} 1 \text{hasChild}.\text{Child}_i \mid 1 \leq i \leq n\} \label{atMost1Childi}\\
n\text{-BoundedTreeNode} &\sqsubseteq \mathord{\leq} n \text{hasChild}.n\text{-BoundedTreeNode} \\
\{\text{Child}_i &\sqsubseteq \exists R_\text{i}.\textsf{Self} \mid 1 \leq i \leq n\} \label{selfConnect} \\
\{R_\text{i} \circ \text{hasParent} \circ \text{hasChild} \circ R_\text{j} &\sqsubseteq \text{hasSibling} \mid 1 \leq i < j \leq n \} \label{chain} 
\end{align}
\caption{Axioms for the $n$-bounded Tree Pattern: These need to be added to the axioms from Figure \ref{figure:tree}}
\label{figure:binaryTreeAxioms}
\end{figure}

Note, though, that hasSibling is non-simple, i.e. the declaration of irreflexivity for hasSibling from Figure \ref{figure:tree} violates regularity. If the ontology shall fall within OWL DL, the irreflexivity axiom should be removed.

Note also, that the approach just spelled out may not be practical for large $n$, as the number of models to be checked, e.g. by a tableaux-based reasoner, will increase exponentially with $n$ due to the disjunction in (\ref{disjunction}).

\section{Conclusions}\label{sec:conclusion}

We have presented a general ontology design pattern for trees together with an axiomatization which makes it possible to answer non-trivial competency questions as they arise in practice. We have also presented a list pattern derived from this tree pattern. We have furthermore discussed limitations of OWL for the modeling of trees, and have provided an alternative axiomatization for the more specific case that the tree is known to be $n$-bounded.

Of course, our approach is still rather straightforward and there exist cases where our model will not suffice. For example, in the application domain discussed in Section \ref{sec:trees-of-life}, it is often desirable to attach additional information to parent-child relationships (i.e., edges), e.g. temporal information. This cannot be done in our current model but would require a reification of the edges using established techniques, and of course this change may affect the treatment of our competency questions. This remains to be investigated.

\bigskip

\noindent\emph{Acknowledgements.} Pascal Hitzler and Hilmar Lapp
acknowledge support by the National Science Foundation (NSF) under
awards 1440202 ``Earthcube Building Blocks: Collaborative Proposal:
GeoLink -- Leveraging Semantics and Linked Data for Data Sharing and
Discovery in the Geosciences'', and DBI-1458484 ``An Ontology-Based
System for Querying Life in a Post-Taxonomic Age'', respectively.

\urlstyle{same}
\bibliographystyle{splncs03}
\bibliography{all}



\end{document}